\RequirePackage{varioref}

\PassOptionsToPackage{table}{xcolor}
\PassOptionsToPackage{capitalise}{cleveref}

\documentclass[sigconf]{acmart}

\def\BibTeX{{\rm B\kern-.05em{\sc i\kern-.025em b}\kern-.08em
    T\kern-.1667em\lower.7ex\hbox{E}\kern-.125emX}}

\setcopyright{none}

\usepackage{adjustbox}
\usepackage{algorithm}
\usepackage{algorithmicx}
\usepackage[noend]{algpseudocode}
\usepackage{amsfonts}
\usepackage{amsmath}
\usepackage{amsthm}
\usepackage{balance}
\usepackage{blindtext}
\usepackage{bm}
\usepackage{booktabs}
\usepackage{color}
\usepackage[T1]{fontenc}    
\usepackage[utf8]{inputenc}
\usepackage{listings}
\usepackage{hyperref}
\usepackage{microtype}
\usepackage{multirow}
\usepackage{nicefrac}
\usepackage{paralist}
\usepackage[capitalise]{cleveref}  

\usepackage[font=footnotesize]{subcaption}

\usepackage{varioref}
\usepackage[table]{xcolor}
\usepackage{xhfill}
\usepackage{xparse}
\usepackage{xpatch}
\usepackage{xspace}

\graphicspath{{./figures/}}

\definecolor{codegreen}{rgb}{0,0.6,0}
\definecolor{codegray}{rgb}{0.5,0.5,0.5}
\definecolor{codepurple}{rgb}{0.58,0,0.82}
\definecolor{backcolour}{rgb}{0.95,0.95,0.92}

\lstdefinestyle{mystyle}{
    backgroundcolor=\color{backcolour},   
    commentstyle=\color{codegreen},
    keywordstyle=\color{magenta},
    numberstyle=\tiny\color{codegray},
    stringstyle=\color{codepurple},
    basicstyle=\footnotesize,
    breakatwhitespace=false,         
    breaklines=true,                 
    captionpos=b,                    
    keepspaces=true,                 
    numbers=left,                    
    numbersep=5pt,                  
    showspaces=false,                
    showstringspaces=false,
    showtabs=false,                  
    tabsize=2
}
\algnewcommand{\LeftComment}[1]{\State \(\triangleright\) #1}

\newcommand{\FISHDBC}{FISHDBC\xspace}
\newcommand{\HDBSCAN}{HDBSCAN*\xspace}

\newcommand{\bigO}{\mathcal O}
\newcommand{\mincs}{\ensuremath{m_\mathit{cs}}\xspace}
\newcommand{\minpts}{\ensuremath{\mathit{MinPts}}\xspace}
\newcommand{\mathify}[1]{\ensuremath{\mathit{#1}}\xspace}
\newcommand{\ef}{\mathify{ef}}
\newcommand{\RG}{\mathify{RG}}

\newcommand{\lzjd}{\texttt{lzjd}\xspace}
\newcommand{\tlsh}{\texttt{tlsh}\xspace}
\newcommand{\sdhash}{\texttt{sdhash}\xspace}

\newcommand{\thmspace}{\FISHDBC's state has size $\bigO(n\log n)$.}
\newcommand{\thmtime}{Adding elements to \FISHDBC and recomputing clustering has average time complexity $\bigO((t + n)\log n)$, where $t$ is the number of calls to $d()$ performed by the HNSW.}
\newcommand{\thmapprox}{The output of \FISHDBC is a valid output of \HDBSCAN run on a distance matrix $D'$ such that $D'_{i,j}=d(i,j)$ if $d(i,j)$ has been called, and $D'_{i,j}=\infty$ otherwise.}

\renewcommand\footnotetextcopyrightpermission[1]{} 
\usepackage{xcolor}
\hypersetup{
    colorlinks,
    linkcolor={red!50!black},
    citecolor={blue!50!black},
    urlcolor={blue!80!black}
}

\begin{document}

\title[\FISHDBC]{\FISHDBC: Flexible, Incremental, Scalable, Hierarchical Density-Based Clustering for Arbitrary Data and Distance}

\author{Matteo Dell'Amico}
\orcid{0000-0003-3152-4993}
\affiliation{Symantec Research Labs}
\email{matteo\_dellamico@symantec.com}

\settopmatter{printacmref=false} 
\settopmatter{printfolios=true} 
\pagestyle{plain} 

\begin{abstract}
\FISHDBC is a flexible, incremental, scalable, and hierarchical density-based clustering algorithm. It is \emph{flexible} because it empowers users to work on arbitrary data, skipping the feature extraction step that usually transforms raw data in numeric arrays letting users define an arbitrary distance function instead. It is \emph{incremental} and \emph{scalable}: it avoids the $\bigO(n^2)$ performance of other approaches in non-metric spaces and requires only lightweight computation to update the clustering when few items are added. It is \emph{hierarchical}: it produces a ``flat'' clustering which can be expanded to a tree structure, so that users can group and/or divide clusters in sub- or super-clusters when data exploration requires so. It is \emph{density-based} and approximates \HDBSCAN, an evolution of DBSCAN.

We evaluate \FISHDBC on 8 datasets, confirming its scalability. Our quality metrics show that \FISHDBC often performs comparably to \HDBSCAN, and sometimes \FISHDBC's results are even preferable thanks to a regularization effect.
\end{abstract}

\maketitle

\section{Introduction}
\label{sec:intro}

In exploratory data analysis (EDA), data are often large, complex, and arrive in a streaming fashion; clustering is an important tool for EDA, because it summarizes datasets---making them more amenable to human analysis---by grouping similar items. Data can be complex because of heterogeneity: consider, e.g., a database of user data as diverse as timestamps, IP addresses, user-generated text, geolocation information, etc. Clustering structure can be complex as well, involving clusters within clusters. Complexity requires clustering algorithms that are \emph{flexible}, in the sense that they can deal with arbitrarily complex data, and are able to discover \emph{hierarchical} clusters. Large datasets call for \emph{scalable} solutions, and streaming data benefits from \emph{incremental} approaches where the clustering can be updated cheaply as new data items arrive. In addition, it is desirable to \emph{distinguish signal from noise} with algorithms that do not fit isolated data items into clusters.

As discussed in Section~\ref{sec:related}, while these problems have been considered previously in the literature,
our proposal tackles \emph{all} of them at once. 
\FISHDBC, which stands for \textit{Flexible, Incremental, Scalable, Hierarchical Density-Based Clustering}, is \textbf{flexible} because it is applicable to \emph{arbitrary data and distance functions}: rather than being forced to convert data to numeric values through a feature extraction process that may lose valuable information, domain experts can encode as much domain knowledge as needed by defining \emph{any} symmetric and possibly non-metric distance function, no matter how complex---our implementation accepts arbitrary Python functions as distance measures. \FISHDBC is \textbf{incremental}: it holds a set of data structures to which new data can be added cheaply and from which clustering can be computed quickly; in a streaming context, new data can be added as they arrive, and clustering can be computed inexpensively. \FISHDBC is also \textbf{scalable}, in the sense that it avoids in most common cases the $\bigO(n^2)$ complexity that most clustering algorithms have when dealing with non-metric spaces; our experiments show that it can scale to millions of data items. It is \textbf{hierarchical}, recognizing clusters within clusters. \FISHDBC belongs to the family of \emph{density-based} algorithms inspired by DBSCAN~\cite{ester1996density}, inheriting the ability to recognize clusters of arbitrary shapes and filtering noise.

\FISHDBC approximates \HDBSCAN~\cite{campello2013density}, an evolution of DBSCAN supporting hierarchical clustering and recognizing clusters with different densities;
\HDBSCAN, however, has $\bigO(n^2)$ computational complexity when using distance functions for which no accelerated indexing exists. The key idea that allows \FISHDBC to be flexible and incremental while maintaining scalability is maintaining a data structure---a spanning tree connecting data items---which is updated as new items are added to the dataset. The problems of neighbor discovery and incremental model maintenance are separated, making the algorithm simpler to understand, implement and modify.
In Section~\ref{sec:algorithm} we present the algorithm, together with an analysis of its time and space complexity and its relationship with \HDBSCAN.

We evaluate \FISHDBC on 8 datasets varying by size, dimensionality, data type, and distance function used. In Section~\ref{sec:experiments}, we validate the scalability and show that clustering quality metrics are often close to the ones of \HDBSCAN, and sometimes they outperform it thanks to a regularization effect. We conclude by discussing when \FISHDBC is preferable to existing approaches in Section~\ref{sec:conclusion}.

\section{Background and Related Work}
\label{sec:background}
\label{sec:related}

Several algorithms have a subset of the desirable properties discussed in \cref{sec:intro}: for example, spectral clustering~\cite{filippone2008survey} is not limited to spherical clusters; agglomerative methods~\cite{murtagh2012algorithms} produce hierarchical clusters and can have incremental implementations. 
To the best of our knowledge, though, no other algorithm embodies at once all the properties that \FISHDBC satisfies, being flexible, incremental, scalable, and providing hierarchical density-based clustering.
Due to space limitations, we cannot cover all approaches that have some of the above properties. In the following, we focus on density-based clustering and approaches applicable to arbitrary data and (potentially non-metric) dissimilarity/distance functions.

\paragraph{Relational Clustering}
These algorithms take as input a \emph{distance matrix} $D$ containing all $\bigO(n^2)$ pairwise distances.  Among them, some are specialized towards arbitrary (non-metric) distances~\cite{laub2004feature,filippone2009dealing}. Unfortunately, these methods are intrinsically not scalable because computing $D$ requires $\Omega(n^2)$ time. \FISHDBC scales better because \emph{not all pairwise distances are computed}: rather than taking a matrix as input, \FISHDBC takes a dataset of arbitrary items and a distance function to apply to them: the distance function will be called on a small subset of the $\bigO(n^2)$ item pairs.

Spectral clustering, which is expensive because it involves factorizing an $\bigO(n^2)$-sized affinity matrix, can be accelerated via the Nyström method~\cite{fowlkes2004nystrom}: computing approximate eigenvectors by randomly sampling matrix rows. This sampling approach would be ineffective for density-based clustering as it would not retrieve a good approximation of each node's local neighborhood, which density-based algorithms need to discover dense areas. \FISHDBC is instead guided by an approximate neighbor search converging towards each node's neighbors, discovering most of them cheaply.

\paragraph{Density-Based Clustering on Arbitrary Data}
Density-based clustering was introduced with DBSCAN \cite{ester1996density} and generalized to arbitrary data in GDBSCAN~\cite{sander1998density}, in which clusters are \emph{connected dense areas}: given a definition of an item's neighborhood (in most cases, given a distance function, the items at distance smaller than a threshold $\varepsilon$), a node is considered to be in a dense area if its neighborhood contains at least \minpts points,%
and each node in its neighborhood is considered to be in the same cluster. In the general case, GDBSCAN has $\bigO(n^2)$ complexity, even though indexing structures can lower the computational complexity of the algorithm, depending on the complexity of range queries~\cite{Schubert:2017:DRR:3129336.3068335} which are $\bigO(n)$ in the general case of arbitrary distance functions. Some subsequent pieces of work still require indexing structures to lower computational complexity~\cite{Mai2018}, while others~\cite{brecheisen2004efficient} are based on filter functions, i.e., cheap functions that return a superset of an item's neighborhood: in this latter case, complexity depends on the filter's function selectivity, i.e., how big their output is. Unlike these approaches, \FISHDBC does not require users to provide an indexing structure or a filter function tailored to the distance function used, and it avoids $\bigO(n^2)$ complexity by introducing approximation.

NG-DBSCAN~\cite{lulli2016ng} is a distributed approximate DBSCAN implementation that discovers neighbors in arbitrary spaces with an approach inspired by the NN-Descent~\cite{dong2011efficient} approximate nearest-neighbor algorithm. Other approaches~\cite{lulli2015scalable,jackson2018scaling} use a similar strategy.
Unlike \FISHDBC, these approaches are not incremental: their results must be wholly recomputed as the dataset changes. Moreover, \FISHDBC benefits from the better scalability of HNSWs over NN-Descent~\cite{10.1007/978-3-319-68474-1_3}. Finally, compared to these works, \FISHDBC inherits the improvements of \HDBSCAN over DBSCAN: better clustering, one less parameter, and hierarchical output.

\paragraph{Incremental Density-Based Clustering}
Unlike our work, existing incremental density-based clustering algorithms~\cite{ester_incremental_1998,kriegel_incremental_2003,fu_ica:_2015} have quadratic complexity in non-metric spaces; moreover, they generally report speed-up factors lower than 100 for incremental recomputation after adding a few elements. What we obtain (see Tables~\ref{tab:synth_runtime} and~\ref{tab:runtime}, ``cluster'' columns) is generally similar or better.

\paragraph{\HDBSCAN}
\citet{campello2013density} improve on DBSCAN while removing the cluster density threshold $\varepsilon$, which is tuned automatically and separately for each cluster. In addition to simplifying tuning, result quality improves because the output can include clusters having different density in the same dataset.

\HDBSCAN introduces the concepts of \emph{core} and \emph{reachability distance}. A node $a$'s core distance $c(a)$ is the distance of its $\minpts\textsuperscript{th}$ closest neighbor, while the reachability distance between items $a$ and $b$ is $\max(d(a, b), c(a), c(b))$ with $d$ being the distance function. Reachability distance essentially factors in the computation the density of each node's neighborhood.
\HDBSCAN computes the minimum spanning tree (MST) $T$ of a complete \emph{reachability graph} $RG$ having data items as nodes and their reachability distance as weights; the hierarchical clustering is obtained from $T$ by removing all edges in order of decreasing weight. Because $T$ is a spanning tree, edge removals split connected components into reciprocally disconnected ones. A \mincs parameter controls the minimum cluster size, and each split is added to the hierarchical clustering if both resulting components have size at least \mincs; \citeauthor{campello2013density} suggest to set $\mincs=\minpts$. The non-hierarchical \emph{flat} output consists of disjoint clusters selected from the hierarchical ones, selecting an $\varepsilon$ threshold for each branch of $T$ to maximize cluster stability across a wide range of densities.
Explicitly computing $RG$ has $\bigO(n^2)$ complexity; \citet{mcinnes2017accelerated} introduced a faster implementation that directly computes $T$ thanks to accelerated lookup structures if the distance function belongs to a set of supported ones.

\section{The \FISHDBC Algorithm}
\label{sec:algorithm}

\HDBSCAN improves on DBSCAN in terms of result quality and by yielding hierarchical results recognizing clusters within clusters. Unfortunately, though, \HDBSCAN is not incremental---if new data arrives, results have to be recomputed from scratch---and it has $\bigO(n^2)$ complexity in the generic case of arbitrary distance functions; it also underperforms when lookup structures are ineffective, e.g., when datasets have very high dimensionality. As our analytic (\cref{sec:proofs}) and empirical (\cref{sec:experiments}) results show, \FISHDBC instead supports incremental computation, maintains or even improves result quality, is accelerated with arbitrary distance functions in most common cases and has a moderate memory footprint.

The core idea of \FISHDBC is maintaining an \emph{approximate} version of the $T$ MST described in \cref{sec:background} and updating it incrementally, at a low cost, as new data arrive. 
We discover candidate edges for $T$ by carefully adapting HNSWs (Hierarchical Navigable Small Worlds~\cite{DBLP:journals/corr/MalkovY16}). HNSWs are indexes conceived for near-neighbor querying in non-metric spaces; however, rather than first building an HNSW representing our dataset and then querying it to find each node's neighbors, we \emph{piggyback} on all calls to the distance function performed by building the index, and generate batches of $\left(a, b, d(a, b)\right)$ triples that we consider for inclusion in $T$.
This strategy allows us to significantly improve \FISHDBC's efficiency because \emph{no} query is ever performed on the HNSW; moreover, we tune the HNSW for speed: as we will see, settings that speed up index construction but would result in low accuracy for nearest-neighbor querying hit desireable trade-offs for our clustering task.

The crux of \FISHDBC's approximation lies in that \emph{not all} $d(a, b)$ pairs are computed, and the clustering result only depends on known distances---as proven in \cref{thm:approximation}, \FISHDBC's results are equivalent to assuming $d(a, b)=\infty$ for non-computed distances. While this may seem to imply a loss in clustering quality, in machine learning~\cite{rudi2015less} and clustering in particular~\cite{han2017mini} subsampling the distance matrix can improve the results by working as a regularization step that avoids overfitting. As discussed in \cref{sec:related}, uniformly sampling the distance matrix would not be effective in our case; hence, we resort to HNSWs which provide a good approximation of a node's neighborhood to estimate local density.

A second regularization effect benefitting \FISHDBC is that there are often multiple valid MSTs of a given reachability graph, because several edges connected to a same node can have the same weight (e.g., because they correspond to that node's reachability distance). \FISHDBC tends to privilege edges towards nodes that are higher up in the HNSW hierarchy, leading to MSTs with a lower diameter (because the top of the HNSW hierarchy is reached more quickly), which in turn corresponds to final outputs with smaller and larger clusters, and with shallower hierarchies. As a consequence of these two factors, some results of \cref{sec:experiments} indeed show that \FISHDBC outperforms \HDBSCAN in terms of quality metrics.

Our implementation is available at \url{https://github.com/matteodellamico/flexible-clustering}.

\subsection{The Algorithm in Detail}

\newcommand{\code}[1]{\ensuremath{\mathtt{#1}}\xspace}
\newcommand{\self}{\code{self}}
\newcommand{\data}{\code{data}}
\newcommand{\mst}{\code{mst}}
\newcommand{\newedges}{\code{candidates}}
\newcommand{\neighbors}{\code{neighbors}}
\newcommand{\HNSW}{\code{HNSW}}
\newcommand{\rd}{\code{rd}}

\begin{algorithm}[t]
\caption{\FISHDBC.}
\label{alg:fishdbc}
\begin{algorithmic}[1]
\Procedure{setup}{$d, \minpts, \ef$} \Comment{$d$ is the distance function}
        \label{algline:setup}
    \State $\self.\minpts \gets \minpts$
    \State $\self.\mst \gets \{\}$ \Comment{approx.\ MST}
        \\\Comment{\mst is a hashtable mapping $(x, y)$ edges to weights}
    \State $\self.\neighbors \gets \{\}$ \Comment{\minpts neighbors per node}
        \\\Comment{maps data to max-heaps of $(\mathrm{distance}, \mathrm{neighbor})$ pairs}
    \State $\self.\HNSW \gets \code{HNSW}(d,  \minpts, \ef)$ 
        \\\Comment{HNSW's $k$ parameter (neighbors per node) is \minpts}
    \State $\self.\newedges \gets \{\}$ \Comment{Candidate edges}
        \\\Comment{mapping of $(x, y)$ edges to weights}
\EndProcedure
\Procedure{add}{$x$} \label{algline:add}
    \State $\self.\HNSW.\code{add}(x)$ \label{algline:add-hnsw}
    \State $\self.\neighbors[x] \gets \minpts\ \text{closest neighbors found}$
    \For{each time $d(x, y)$ is called by \HNSW returning $v$} \label{algline:add-outer}
        \State $\rd \gets \max(v, \text{core distances of $x$ and $y$})$
        \State $\self.\newedges[x, y] \gets \rd$ \Comment{Reachability distance}
            \label{algline:add-everything}
        \If{we found a new top-\minpts neighbor for $y$}
            \State update $\self.\neighbors[y]$
            \ForAll{neighbor $z$ of $y$ at distance $w < v$} \label{algline:add-inner}
                \If{core distance of $z$ is less than $v$}
                    \State $\rd \gets \max(w, \text{core distances of $y$ and $z$})$
                    \State $\newedges[y, z] \gets \rd \label{algline:innerlookup1}$ \label{algline:update}
                        \\\Comment{reachability distance for $(y, z)$ decreased}
                \EndIf
            \EndFor \label{algline:add-end-inner}
        \EndIf
    \EndFor
    \If{$|\newedges| > \alpha \cdot |\neighbors|$}
        \label{line:newedges_size}
        call \textsc{update\_MST}
        \\\Comment{We guarantee that \newedges has $\bigO(n)$ size}
    \EndIf
\EndProcedure
 \Procedure{update\_MST}{} 
    \label{algline:update-mst}
    \State $\mst \gets \mathify{Kruskal}(\mst \cup \newedges)$
    \State $\newedges \gets \{\}$
\EndProcedure
\Function{cluster}{\mincs} \label{algline:cluster}
    \If {\newedges is not empty}
        call \textsc{update\_MST}
    \EndIf
    \State compute clustering from MST
        \\\Comment{using \citet{mcinnes2017accelerated}'s approach}
\EndFunction
\end{algorithmic}
\end{algorithm}

\Cref{alg:fishdbc} shows \FISHDBC in pseudocode.
The state consists of four objects:
\begin{inparaenum}
\item the HNSW;
\item \neighbors: each node's \minpts closest discovered neighbors and their distance;
\item the current approximated MST and, for each edge $(a, b)$ in it, the corresponding value of $d(a, b)$;
\item \newedges, a temporary collection of candidate MST edges.
\end{inparaenum}
\textsc{Setup}
initializes the state.

\textsc{Add} is called to incrementally add
a new element $x$ to the dataset. It adds $x$ to the HNSW, updates the max-heap of $x$'s neighbors with those discovered in the HNSW, and then processes all the pairs $(x, y)$ whose distance has been computed while adding $x$ to the HNSW. Each of them is considered as a candidate edge for our MST; in addition, we add to the candidate MST edges \newedges all those for which the reachability distance decreased due to the new edge. 
Since \neighbors contains max-heaps,
each item's core distance---i.e., the distance of the $m^\textsuperscript{th}$
closest neighbor---is accessible 
at the top of the heap.
If \newedges became larger than $\alpha n$, we call \textsc{update\_MST} to free memory. $\alpha$ has a moderate impact on runtime, and should be chosen as large as possible while guaranteeing that \FISHDBC's state will fit in memory.

\textsc{Update\_MST} processes the temporary set of candidate edges \newedges.
Any minimum spanning forest algorithm can be called on the union of the current MST and the new candidates; in our implementation, we use Kruskal's algorithm.
Technically, the approximate MST might be a forest---an acyclic graph with multiple connected components---rather than a tree; as shown in \cref{thm:approximation}, this has no effect on final results.
In a streaming context when data arrives incrementally, this procedure can be called during idle time.

The output is finally computed using the bottom-up strategy by \citet{mcinnes2017accelerated} after calling \textsc{update\_MST}.

\paragraph*{About HNSWs and the \FISHDBC Design}

HNSWs represent each dataset as a set of layered approximated $k$-nearest neighbor graphs, where the bottom layer contains the whole dataset, and each other one contains approximately $1/k$-th of the elements in the layer below it. Neighbors are found through searches starting at the top layer and continuing in the lower ones when a local minimum is found in the above layer. Since we want to find the \minpts nearest neighbors, we set $k=\minpts$. The \ef parameter controls the effort spent in the search; in Section~\ref{sec:experiments} we show that $\ef\in[20, 50]$ yields a good trade-off between speed and quality of results.

One may think that \FISHDBC could have a simpler design,
computing the MST based on the nearest neighbor distances in the bottom graph of the HNSW which represents the whole dataset, similarly to other approaches~\cite{lulli2015scalable,jackson2018scaling}. This, however, is not optimal as information about farther away items is important to avoid breaking up large clusters: often, small clusters having around close to \minpts nodes are disconnected from other (close) clusters in the nearest neighbor graph. By gradually converging towards closest nodes during neighbor search, we obtain enough information about other nodes to ensure that local clusters remain connected.

\subsection{Properties of \FISHDBC}
\label{sec:proofs}

We now give proofs relative to \FISHDBC's complexity in terms of space and time, as well as studying its relationship with \HDBSCAN.

\paragraph{Space Complexity}
The asymptotic memory footprint of \FISHDBC is rather small: this is confirmed in \cref{sec:experiments}, where we show that \FISHDBC can handle datasets that are too large for \HDBSCAN.

\begin{theorem}
\label{thm:space}
\thmspace
\end{theorem}

\begin{proof}
\FISHDBC's state consists of
\begin{inparaenum}
\item the HNSW ($\bigO(n\log n)$ size~\cite{DBLP:journals/corr/MalkovY16});
\item \neighbors: each node's \minpts closest discovered neighbors and their distance ($\bigO(n)$ size);
\item \mst: the current approximated MST stored as a mapping between edges and their weight ($n$ nodes and at most $n-1$ edges, hence $\bigO(n)$ size);
\item the temporary set \newedges of candidate edges ($\bigO(n)$ size, because each call to \textsc{add} will add to \newedges at most $n-1$ elements).
\end{inparaenum}
The union of these four objects has therefore size $\bigO(n\log n)$.
\end{proof}

\paragraph{Time Complexity}
This theorem justifies why computation time grows slowly as dataset size increases (e.g., \cref{fig:finefoods_building}).

\begin{theorem}
\label{thm:time}
\thmtime
\end{theorem}

The time complexity of \FISHDBC of depends on HNSWs: if they require few distance calls, computation cost remains low. We experimentally see that this is true in most real-world cases; moreover, \citet{DBLP:journals/corr/MalkovY16} show that HNSWs have $t=\bigO(l\log n)$ for adding $l$ elements under some assumptions. \citeauthor{DBLP:journals/corr/MalkovY16} provide experimental results that support this, similarly to our own results which also show a coherent behavior. When this holds, incrementally processing $l$ elements has time complexity $\bigO(l \log^2 n + n \log n)$, and processing a whole dataset has complexity $\bigO(n \log^2 n)$. Our experiments show that most computation is spent in incrementally building and updating the MST, while computing clustering is orders of magnitude cheaper (e.g., \cref{tab:synth_runtime}).

\begin{proof}
We will call \textsc{add}$(x)$ for each new element $x$ to update the model, and then \textsc{cluster} to obtain the clustering.

Core distance lookups have $\bigO(1)$ cost as they are accessible at the top of each heap in \neighbors.
The complexity of adding elements to the HNSW is $\bigO(t)$ where $t$ is the number of calls to $d()$.
In the rest of the \textsc{add} procedure (see \cref{alg:fishdbc}), the most computationally intensive part is the inner loop of lines~\ref{algline:add-inner}--\ref{algline:add-end-inner}.
This loop is executed at most $\bigO(t\minpts)$ times: the $\bigO(t)$ factor is due to the outer loop (line \ref{algline:add-outer}) and $\bigO(\minpts)$ to the inner loop. The hashtable lookup at line~\ref{algline:innerlookup1} 
has
complexity $\bigO(1)$, for an average complexity of $\bigO(t\minpts)$ for the whole time spent in the \textsc{add} procedure, excluding \textsc{update\_MST} calls.

The cost of \textsc{update\_MST} is determined by the MSF algorithm. Kruskal's algorithm, which we use, has time complexity $\bigO(E\log E)$ where $E$ is the number of input edges. Since $E\in\bigO(n)$ here, a call of \textsc{update\_MST} has cost $\bigO(n \log n)$. This function will be called $\bigO(t/n + 1)$ times, resulting in a computational complexity of $\bigO(t/n + 1)n\log n=\bigO((t+n)\log n)$ for this procedure.

The call to \textsc{cluster} has complexity $\bigO(n\log n)$~\cite{mcinnes2017accelerated}.

The dominant cost is the time spent in \textsc{update\_MST}, yielding a total complexity of $\bigO((t+n)\log n)$.
\end{proof}

\paragraph{Approximation of \HDBSCAN}

We show that the only reason for the approximation is that we do not compute all pairwise distances: \FISHDBC computes a valid result of \HDBSCAN when the latter is passed a distance matrix in which all the pairwise distances that are not computed are set to infinity.
If $d()$ is called on all the $\bigO(n^2)$ pairwise distances, we will indeed be proving that \FISHDBC is equivalent to \HDBSCAN.

We first prove that, in a reachability graph, edges with weight $\infty$ can be safely removed without any effect on the resulting clustering.

\begin{lemma}
\label{lemma:infty}
Consider two reachability graphs \RG and \RG', where \RG' is obtained by removing all edges weighted $\infty$ from \RG. Clusterings resulting from \RG and \RG' are equivalent.

\begin{proof}
The procedure we use to compute clustering~\cite{mcinnes2017accelerated} starts by considering each node as a cluster, iterates through MST edges grouped by increasing weight, and joins in the same cluster the nodes connected by those edges. When clusters of size at least \mincs are joined, they are added to the hierarchical clustering---excluding the root cluster which contains all nodes.

Let us consider the minimum spanning forests $F$ and $F'$ obtained respectively from \RG and \RG'. Because \RG is a full graph, $F$ is a spanning tree, while $F'$ may not be. If $F=F'$, the thesis is proven. If $F\neq F'$, it must be because all edges of $F'$ are present in $F$, and one or more edges having weight $\infty$ are present in $F$.
Since edges of the MST are processed by increasing weight, these $\infty$-weighted edges are processed last, hence the output for $F$ and $F'$ will be the same until then; joining edges in this last step will necessarily result in the root cluster containing all nodes which is not returned in the final results. The two outputs will therefore be the same.
\end{proof}
\end{lemma}

We can now prove our theorem.

\begin{theorem}
\label{thm:approximation}
\thmapprox

\begin{proof}
\HDBSCAN can have several valid outputs because it is based on computing a spanning tree of the reachability graph, which may not be unique if several edges have the same weight. We prove the equivalence for at least one of the valid spanning trees.

We base ourselves on a result by \citet[Lemma 1]{eppstein_offline_1994}, which proves that minimum spanning forests (MSFs) can be built incrementally: rather than taking as input a whole graph $G$ at once we can take a subgraph $G'$, compute its MSF $F'$ and ignore the rest of $G'$. We can later add to $F'$ the parts of $G$ that were not in $G'$ and compute an MSF of the resulting graph: it will be a correct MSF $\hat F$ of $G$. Hence, we can add edges incrementally in batches and keep memory consumption low (while $G$ has size $\bigO(n^2)$, $F$ has size $\bigO(n)$). More formally, given a graph $G=(V,E)$ and a subgraph of it $G'=(V' \subseteq V, E' \subseteq E)$, for every MSF $F'$ of $G'$, there exists an MSF $\hat F$ of $G$ such that $(E' \setminus F') \cap \hat F = \emptyset$.

Given the reachability graph \RG obtained from $D'$ we consider \RG', which is \RG without all the edges having weight $\infty$. 
Due to Lemma~\ref{lemma:infty}, our goal reduces to showing that \FISHDBC will end up having in \mst a minimum spanning forest of \RG'.

Recall the \textsc{update\_\allowbreak MST} procedure 
of Algorithm~\ref{alg:fishdbc}: we iteratively add elements from \newedges to \mst and discard the edges that are not part of the MSF. Thanks to the aforementioned result by \citeauthor{eppstein_offline_1994}, our thesis is proven if all edges of \RG' eventually end up in \newedges: this is actually done in line~\ref{algline:add-everything}; the reachability distance might not be correct if some neighbors are not yet known, but this will be eventually updated to the correct value (line~\ref{algline:update-mst}) when neighbors are discovered. We may include a single edge multiple times in \newedges, but the weight always decreases: since we compute a \emph{minimum} spanning forest, only the last (and correct) value for the weight will end up in \mst at last.

Since all edges of $\RG'$ are eventually added to \newedges with their correct weights, \mst will be a minimum spanning forest of \RG', which thanks to Lemma~\ref{lemma:infty} proves our thesis.\end{proof}
\end{theorem}

\begin{table*}[htbp]
    \centering
    \begin{adjustbox}{max width=\textwidth}
    \rowcolors{3}{gray!15}{white}
    \begin{tabular}{lrlllllrrll}
    \toprule
    \multirow{2}{*}{Dataset} & \multirow{2}{*}{Size} & \multirow{2}{*}{Data type} & \multirow{2}{*}{Distance function(s)} & \multirow{2}{*}{Metric} & \multirow{2}{*}{Labeled} & \multicolumn{2}{c}{Results}\\
    & & & & & & Quality & Runtime \\
    \midrule
    Blobs & 10\ 000 & 1,000 to 10,000-d vectors & Euclidean & yes & yes & \cref{tab:blobs-ext} & \cref{fig:blobs_runtime} \\
    DW-Enron & 39\ 861 & Sparse 914-d vectors & cosine & no & no & \cref{tab:internal} & \cref{tab:runtime}\\
    DW-NYTimes & 300\ 000 & Sparse 2,120-d vectors & cosine & no & no & \cref{tab:internal} & \cref{tab:runtime}\\
    Finefoods & 568\ 474 & Text (average 430 chars) & Jaro-Winkler & no & no & \cref{tab:internal} & \cref{tab:runtime}\\
    Fuzzy hashes & 15\ 402 & File digests &  \lzjd, \tlsh, \sdhash & no & yes & \cref{fig:fuzzy_runtime} & \cref{tab:fuzzy_external} \\
    Household & 2\ 049\ 280 & 7-d vectors & Euclidean & yes & no & \cref{tab:internal} & \cref{tab:runtime} \\
    Synth & 10\ 000 & 640--2,048-d sparse bool vectors & Jaccard & yes & yes & \cref{tab:synth_fid} & \cref{tab:synth_runtime} \\
    USPS & 2\ 197 & 16x16 bitmaps & Simpson score & no & yes & \cref{tab:usps-ext} & \cref{tab:runtime} \\
    \bottomrule
    \end{tabular}
    \end{adjustbox}
    \caption{Evaluated datasets.}
    \label{tab:datasets}
\end{table*}

\section{Experimental Evaluation}
\label{sec:experiments}

The key novelties of \FISHDBC with respect to \HDBSCAN are incremental implementation and
handling arbitrary data and distance functions while maintaining scalability. \HDBSCAN is regarded as an improvement on DBSCAN and known for the result quality~\cite{campello2013density,Schubert:2017:DRR:3129336.3068335}, and the accelerated implementation by \citet{mcinnes2017hdbscan} is competitive in terms of runtime with many other algorithms~\cite{mcinnes2017accelerated}.
In the following, we therefore use \citet{mcinnes2017hdbscan}'s \HDBSCAN implementation as a strong state-of-the-art baseline for both speed and clustering quality which also handles arbitrary data and distance functions and returns hierarchical results, and evaluate where \FISHDBC does (and \textit{does not}) outperform it. We refer to \citet{mcinnes2017accelerated} for comparisons between our reference \HDBSCAN implementation and other algorithms. We consider comparisons against distributed DBSCAN implementations~\cite{lulli2016ng,Song:2018:RSP:3183713.3196887} as out of scope, also because of the difficulties in performing fair comparisons between single-machine and distributed approaches~\cite{mcsherry2015scalability}.

\subsection{Experimental Setup}
\label{sec:datasets}

The goal is to test \FISHDBC's flexibility by evaluating it on several very diverse datasets and distance functions. We evaluate \FISHDBC's quality/runtime tradeoff on a single machine with 128 GB of RAM and different values of the \ef HNSW parameter: 20 for faster computation and, in some cases, lower quality, and 50 for slower computation and possibly better results. We performed experiments---reported where space allows---with other values ($\ef\in[10,200]$), which hit less desireable tradeoffs: this is remarkable, because \citet{DBLP:journals/corr/MalkovY16} report a good tradeoff between speed and approximation with a value of $\ef=100$ for their problem of nearest neighbor search; in our clustering use case, we can significantly cut computation without large impacts on result quality by choosing lower values of \ef.
Following the advice of \citet{Schubert:2017:DRR:3129336.3068335}, we use a low value of $\minpts=10$; in additional experiments---not included due to space limitations---we see that \minpts has only a minor effect on final results. HNSW parameters are set to the defaults of \citet{DBLP:journals/corr/MalkovY16}, except for \ef.

\paragraph{Datasets}

We validate \FISHDBC on 8 datasets and 8 different distance functions (\cref{tab:datasets}). While many related works are evaluated on large datasets with only a handful of dimensions, we are especially interested in high-dimensional cases, where ad-hoc lookup structures (and algorithms based on them) often do not scale well.

\begin{asparadesc}
\item{\bf Blobs.} Synthetic labeled datasets of isotropic Gaussian blobs (10 centers, 10,000 samples) generated with scikit-learn~\cite{scikit-learn}. Results are averaged over generated datasets; the standard deviation is small enough that it would not be discernible in plots.
\item{\bf Docword.} The DW-* datasets~\cite{asuncion2007uci} represent text documents as high-dimensional bags of words; here, we use cosine distance.
\item{\bf Finefoods} consists of unlabeled textual food reviews~\cite{mcauley2013amateurs}, which we cluster with the Jaro-Winkler edit distance~\cite{winkler1999state}.
\item{\bf Fuzzy Hashes} are digests of binary files from the study of \citet{pagani2018beyond}---digests can be compared to output a similarity score between files. We use three algorithms: \lzjd~\cite{raff2018lempel}, \sdhash~\cite{breitinger2012security} 
and \tlsh~\cite{oliver2013tlsh}. \sdhash and \tlsh have been evaluated as sound approaches by \citeauthor{pagani2018beyond}, while \lzjd is a recent improvement~\cite{raff2018lempel}. Files have 5 labels each: program, package, version, compiler used to build it, and options passed to the compiler.
\item{\bf Household} is a large unlabeled 7-dimensional dataset of power consumption data~\cite{asuncion2007uci}. We use Euclidean distance.
\item{\bf Synth} datasets 
are created with \citet{cesario2007top}'s generator, simulating transactions as event sets. 
In each, we generate 5 clusters of transactions with no outliers, no overlapping and dimensionality varying between 640 and 2,048.
We use Jaccard distance.
\item{\bf USPS.} A set of 16x16-pixel images of handwritten letters~\cite{laub2004feature}. Like other works~\cite{laub2004feature,filippone2009dealing},
we consider the 0 and 7 digits and
discretize them to a bitmap using a threshold of 0.5, and we consider only those with at least 20 pixels having a value of 1, for a total of 2,196 elements. As in these works, 
we use the Simpson score as our distance function. Where $\mathtt{\&}$ is the bitwise-and function and $c()$ is the function that returns the number of `1' bits, the Simpson distance between bitmaps $x$ and $y$ is $1-c(x\mathtt{\&}y)/\min(c(x), c(y)).$
\end{asparadesc}

\begin{figure*}[thbp]
    \newcommand{\width}{.32\textwidth}
    \centering
    \begin{subfigure}[t]{\width}
        \includegraphics[width=\textwidth]{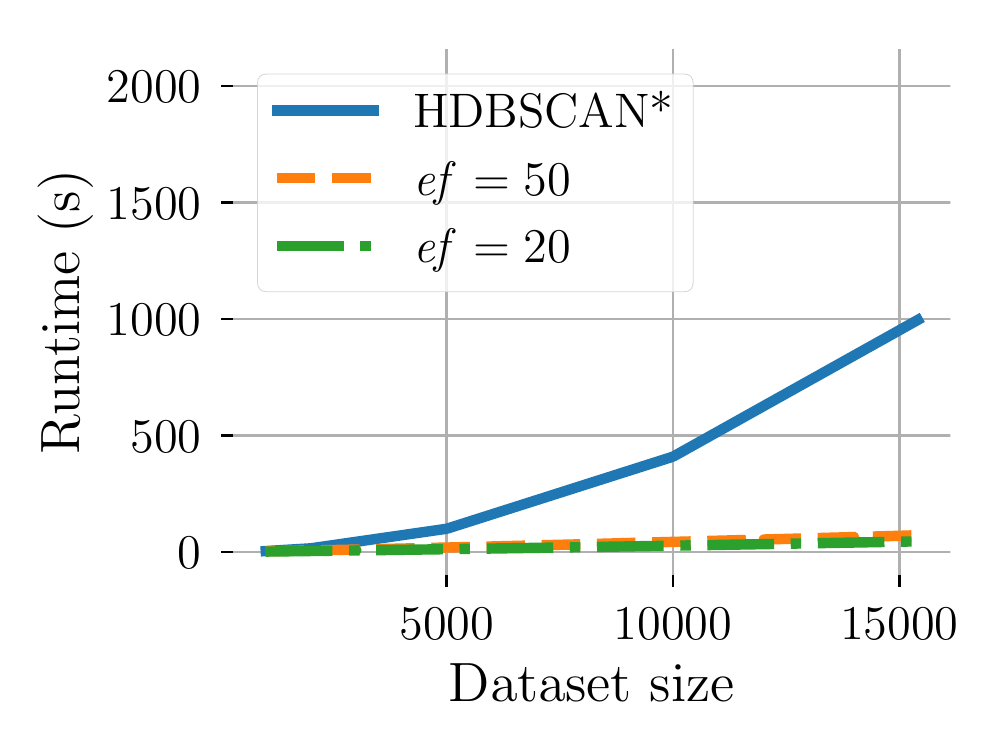}
        \caption{\lzjd.}
        \label{fig:fuzzy-runtime-lzjd}
    \end{subfigure}
    \begin{subfigure}[t]{\width}
        \includegraphics[width=\textwidth]{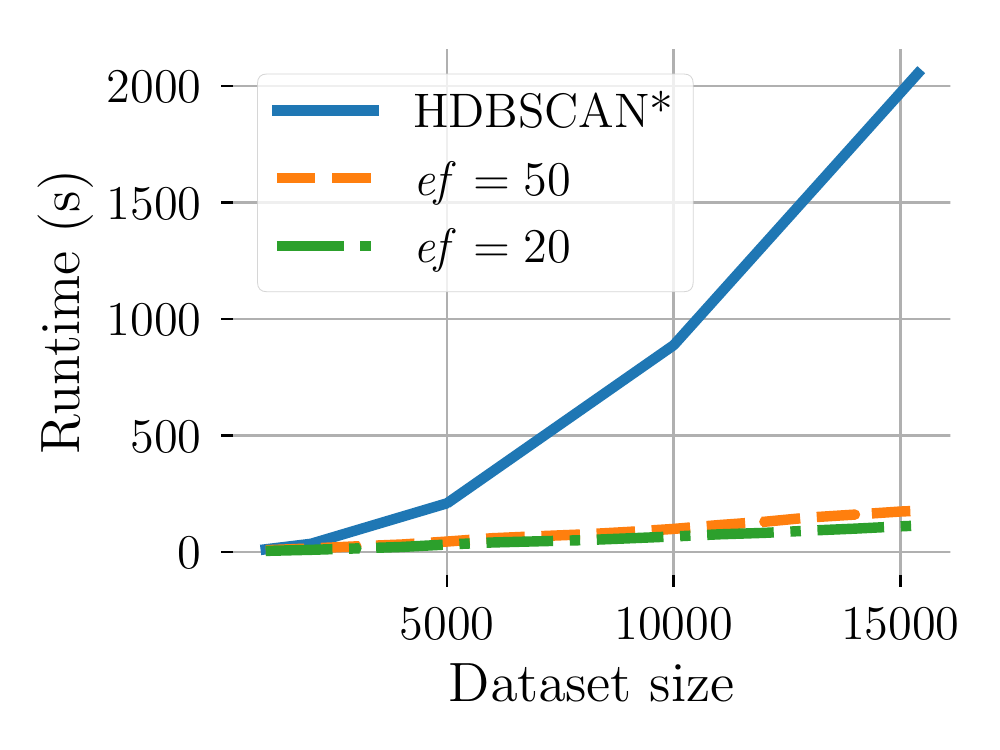}
        \caption{\sdhash.}
    \end{subfigure}
    \unboldmath
    \begin{subfigure}[t]{\width}
        \includegraphics[width=\textwidth]{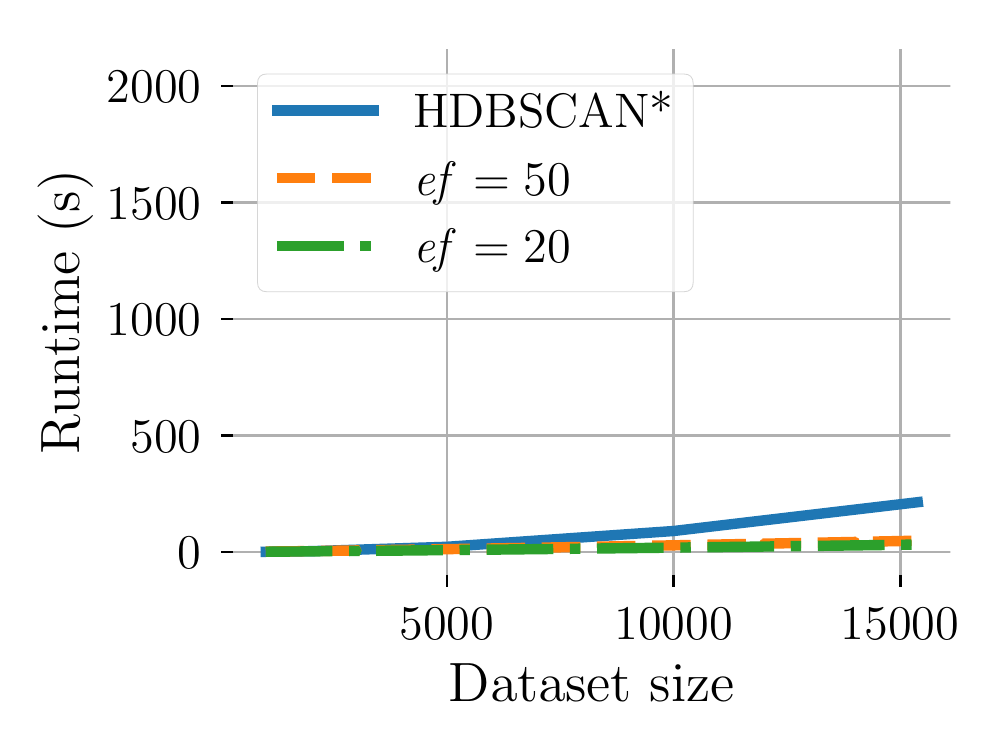}
        \caption{\tlsh.}
        \label{fig:fuzzy-runtime-tlsh}
    \end{subfigure}
    \caption{Fuzzy hashes: runtime comparison. The $\ef=\{20,50\}$ parameter is passed to the HNSW.}
    \label{fig:fuzzy_runtime}
\end{figure*}


\begin{table*}[htbp]
\centering
\begin{tabular}{llrrrrrrrrrrrr}
\toprule
\multirow{2}{*}{Fuzzy hash} & Clustering & \multirow{2}{*}{\ef} & \# clustered & \multicolumn{2}{c}{Program} & \multicolumn{2}{c}{Package} & \multicolumn{2}{c}{Version} & \multicolumn{2}{c}{Compiler} & \multicolumn{2}{c}{Options} \\
& algorithm && elements & AMI & AMI* & AMI & AMI* & AMI & AMI* & AMI & AMI* & AMI & AMI* \\
\midrule
\multirow{3}{*}{\lzjd} & \multirow{2}{*}{\FISHDBC} & 20 & 12\ 710 & 0.47 & 0.42 & \bf 0.14 & \bf 0.13 & \bf 0.10 & \bf 0.09 & \bf 0.14 & \bf 0.13 & \bf 0.16 & \bf 0.15 \\
&& 50 & 12\ 879 & \bf 0.48 & \bf 0.43 & \bf 0.14 & \bf 0.13 & 0.09 & 0.08 & \bf 0.14 & \bf 0.13 & \bf 0.16 & \bf 0.15 \\
& \HDBSCAN & & \bf 13\ 365 & 0.47 & \bf 0.43 & \bf 0.14 & \bf 0.13 & 0.09 & 0.08 & \bf 0.14 & \bf 0.13 & \bf 0.16 & \bf 0.15 \\
\midrule
\multirow{3}{*}{\sdhash} & \multirow{2}{*}{\FISHDBC} & 20 & 6\ 905 & 0.52 & 0.21 & \bf 0.14 & 0.11 & \bf 0.09 & 0.07 & \bf 0.15 & 0.12 & \bf 0.15 & 0.12 \\
&& 50 & 9\ 614 & \bf 0.53 & 0.35 & \bf 0.14 & \bf 0.12 & \bf 0.09 & \bf 0.08 & \bf 0.15 & \bf 0.13 & 0.14 & \bf 0.13 \\
& \HDBSCAN && \bf 13\ 184 & 0.46 & \bf 0.42 & 0.12 & 0.11 & 0.07 & 0.07 & 0.12 & 0.11 & 0.12 & 0.11 \\
\midrule
\multirow{3}{*}{\tlsh} & \multirow{2}{*}{\FISHDBC} & 20 & 9\ 746 & \bf 0.46 & 0.30 & \bf 0.18 & \bf 0.16 & \bf 0.13 & \bf 0.12 & \bf 0.17 & \bf 0.15 & \bf 0.19 & \bf 0.16 \\
&  & 50 & 10\ 046 & \bf 0.46 & \bf 0.32 & \bf 0.18 & \bf 0.16 & \bf 0.13 & \bf 0.12 & \bf 0.17 & \bf 0.15 & 0.18 & \bf 0.16 \\
& \HDBSCAN && \bf 12\ 958 & 0.34 & 0.31 & 0.13 & 0.13 & 0.10 & 0.09 & 0.13 & 0.13 & 0.15 & 0.14 \\
\bottomrule
\end{tabular}
\caption{Fuzzy hashes: external quality metrics, applied to different distance functions (rows) and labels (columns).}
\label{tab:fuzzy_external}
\end{table*}

\paragraph{Quality metrics}

We evaluate clustering on labeled data\-sets with external metrics: adjusted mutual information (AMI) and adjusted Rand index (ARI). These metrics 
vary between 0 (random clustering) and 1 (perfect matching). Like most density-based clustering algorithms, \FISHDBC does not cluster \emph{all} the elements, returning instead a set of unclustered ``noise'' elements: for this reason, we compute AMI and ARI by taking into account only the clustered elements. A metric like this, however, may reward clusterings that only group extremely similar items and mark as noise the rest of the dataset: hence, we use two additional metrics---respectively, AMI* and ARI*---that consider all noise items as a single additional cluster. While AMI/ARI evaluate whether \textit{clustered} elements are grouped similarly to the reference labeling, AMI*/ARI* penalize outputs that do not cluster many items. Other options can be envisioned, such as treating each noise item as a single cluster, but this could trigger known problems as metrics such as AMI are biased against solutions with many small clusters~\cite{gates2017impact}.
\citet{romano2016adjusting} advise using AMI rather than ARI for unbalanced datasets; as this can be the case when some clusters are disproportionately recognized as noise, we always use AMI and include ARI when space allows it.

For unlabeled datasets, we resort to internal metrics, such as silhouette, intra- and inter-cluster distance~\cite{liu2010understanding}. Silhouette 
is expensive to compute and generally requires more memory than \FISHDBC, hence we obtained out-of-memory errors (\textit{OOM}) on larger datasets; for intra-cluster (lower is better) and inter-cluster distance (higher is better) we resorted, for the larger clusters, to sampling, choosing two random elements from the same cluster (intra-cluster) or different clusters (inter-cluster), normalizing the probability of choosing each cluster to ensure that each pair has the same probability of being selected. We use a sample size of 10,000.

We do not use the density-based clustering validation metric by \citet{moulavi2014density}, as---besides having $\bigO(n^2)$ complexity---it is designed for low-dimensional datasets: results are unstable and overflow in our case because distances are exponentiated by the number of dimensions.

\subsection{Ad-Hoc Distance Measures}

We now consider distance measures for which our reference \HDBSCAN implementation~\cite{mcinnes2017hdbscan} does not provide accelerated support; in such cases, it is still possible to run \HDBSCAN by computing a pairwise distances matrix. Here, \FISHDBC can scale better than \HDBSCAN because of the lower asymptotical complexity.

\paragraph{Fuzzy Hashes}

This dataset has the interesting property of having overlapping class labels. We start by analyzing \cref{fig:fuzzy_runtime}: here, computational cost is dominated by the calls to the distance function, and we clearly see a quadratic increase in runtime for \HDBSCAN---differences between \HDBSCAN results are essentially due to the differences in cost between the distance functions. \FISHDBC consistently scales much better than \HDBSCAN.

The quality metrics of \cref{tab:fuzzy_external}, where we evaluate AMI and AMI* for each fuzzy hash algorithm/labeling pair, inspire some considerations.

First, \HDBSCAN consistently clusters more files than \FISHDBC, but the AMI score of \FISHDBC is often higher. This means that \FISHDBC identifies more elements as noise, while outputting the other elements in more coherent clusters.

Second (with the single exception of \sdhash applied to the ``program'' label where \FISHDBC's approximation appears to impact result quality negatively), the AMI* scores of \HDBSCAN are generally equivalent or worse than those of \FISHDBC, suggesting that the additional elements clustered by \HDBSCAN are often not well clustered. This can be explained by the argument of \cref{sec:algorithm}, which suggests that---by working as regularization---\FISHDBC's approximation can improve output quality. By manually examining results, we confirm that the hierarchical clustering of \FISHDBC is generally simpler, with fewer larger clusters and a shallower hierarchy.

\begin{table}[htbp]
    \centering
    \rowcolors{3}{gray!15}{white}
    \begin{tabular}{rrrrrr}
    \toprule
    \multirow{2}{*}{dim} & \multicolumn{2}{c}{\FISHDBC ($\ef=20$)} & \multicolumn{2}{c}{\FISHDBC ($\ef=50$)} & \multirow{2}{*}{\HDBSCAN} \\
    \cmidrule(lr){2-3}\cmidrule(lr){4-5}
    & build & cluster & build & cluster & \\
    \midrule
    640 & \bf 67.5 & 0.21 & 109 & 0.24 & 115 \\
    1\ 024 & \bf 65.7 & 0.19 & 103 & 0.20 & 100 \\
    2\ 048 & \bf 82.2 & 0.22 & 126 & 0.23 & 155 \\
    \bottomrule
    \end{tabular}
    \caption{Synth: runtime (s). ``Build'' is the time to incrementally build the \FISHDBC data structures, ``cluster'' the time to compute clustering using them as input.}
    \label{tab:synth_runtime}
\end{table}

\begin{table}[htbp]
\centering
    \rowcolors{3}{gray!15}{white}
    \begin{tabular}{rrrrrrr}
    \toprule
    \multirow{2}{*}{\ef} & \multicolumn{2}{c}{$\mathrm{dim}=640$} &\multicolumn{2}{c}{$\mathrm{dim}=1,024$} & \multicolumn{2}{c}{$\mathrm{dim}=2,048$} \\
    & AMI* & ARI* & AMI* & ARI* & AMI* & ARI* \\
    \midrule
    20 & 0.89 & 0.94 & 0.96 & 0.99 & \bf 1 & \bf 1 \\
    50 & \bf 0.96 & \bf 0.98 & \bf 0.96 & \bf 0.99 & \bf 1 & \bf 1 \\
    \midrule
    \HDBSCAN & 0.49 & 0.75 & 0.79 & 0.95 & \bf 1 & \bf 1 \\
    \bottomrule
    \end{tabular}
    \caption{Synth: external quality metrics.}
    \label{tab:synth_fid}
\end{table}

\paragraph{Synth}

\Cref{tab:synth_runtime} reports on runtime while varying \ef. \FISHDBC spends most of the time building incrementally its data structures, while the cost of extracting a clustering from them is more than two orders of magnitude cheaper. Therefore clustering can be recomputed, cheaply, as the data structure grows; as shown in \cref{tab:runtime}, this is the case in all our datasets.
\FISHDBC outperforms \HDBSCAN here, with a margin growing as the dimensionality (and hence the cost of the distance function) grows. Compared to the Fuzzy Hashes dataset, the smaller difference is largely due to a cheaper distance function.
Quality results in \cref{tab:synth_fid} are perhaps more surprising: for 640 and 1,024 dimensions, \FISHDBC substantially outperforms \HDBSCAN; once again, we attribute this to the regularization effect described in Section~\ref{sec:algorithm}. As the dimensionality grows, 
clusters become more separated and quality metrics values grow.

\begin{figure}[t]
    \centering
        \includegraphics[width=\columnwidth]{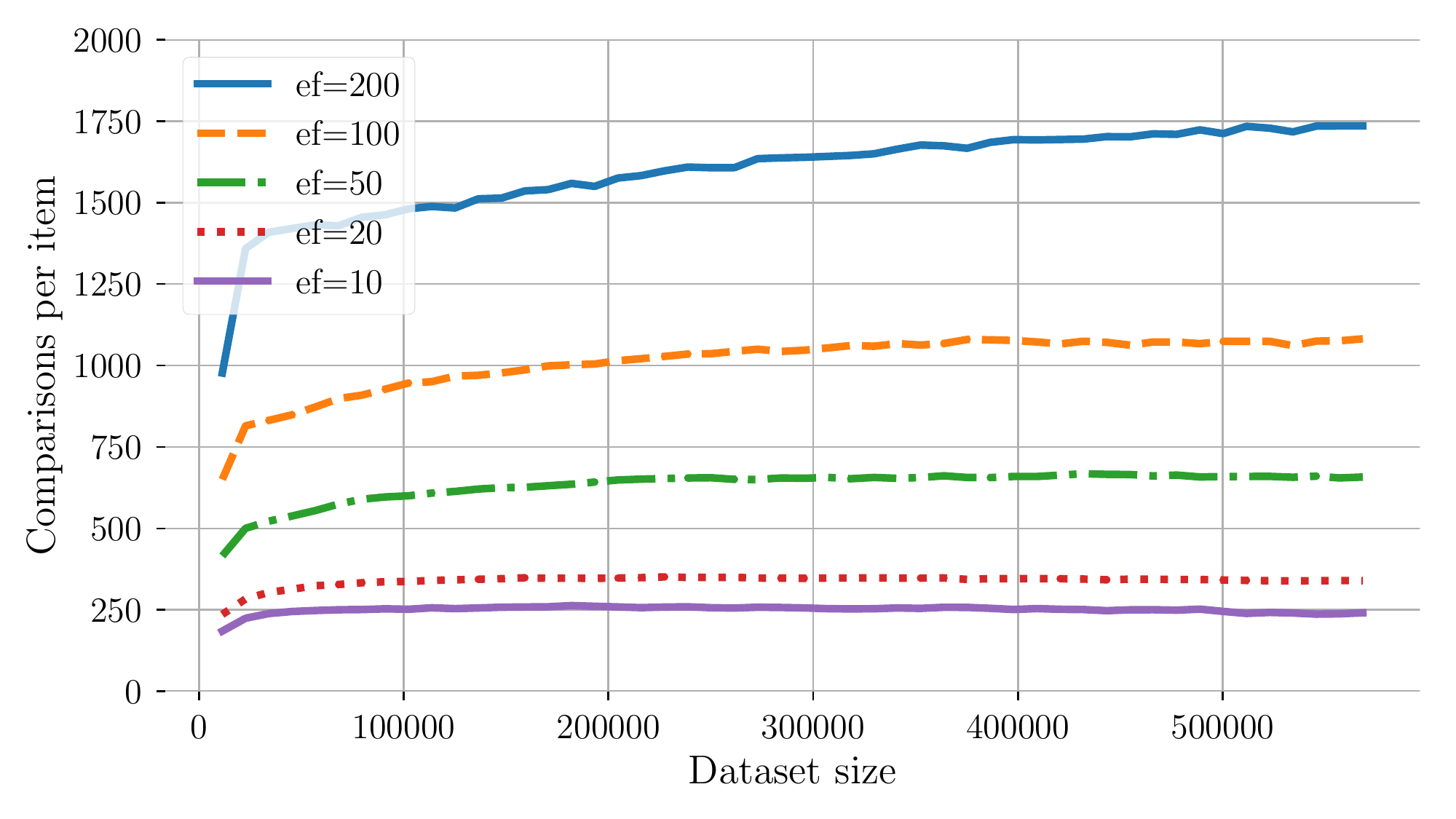}
    \caption{Finefoods: scalability as the dataset size increases.}
    \label{fig:finefoods_building}
\end{figure}

\paragraph{Finefoods}

This dataset is rather large, and the Jaro-Winkler distance applied to it is quite expensive. We could not apply \HDBSCAN to this dataset, as the full distance matrix would be very expensive to compute and could not fit in memory; this dataset allows us investigate \FISHDBC's scalability. In \cref{fig:finefoods_building} we observe the average number of calls to the distance function performed per item as new elements get introduced in the \FISHDBC data structure (a clustering is computed every time 2\% of the dataset is added). We can see that, in the beginning, the number of comparisons grows as the dataset does, but it tends to plateau afterwards.
Results for quality metrics and runtime are available in Tables~\ref{tab:internal} and~\ref{tab:runtime}.

\begin{table}[t]
    \centering
    \rowcolors{2}{gray!15}{white}
    \begin{tabular}{rrrrrrrr}
    \toprule
    \ef & \# clustered & AMI & AMI* & ARI & ARI* \\
    \midrule
    20 & \bf 1\ 334 & \bf 1 & \bf 0.41 & \bf 1 & \bf 0.41 \\
    50 & 1\ 307 & \bf 1 & 0.40 & \bf 1 & 0.40 \\
    \midrule
    \HDBSCAN & 1\ 102 & 0.53 & 0.25 & 0.59 & 0.20 \\
    \bottomrule
    \end{tabular}
    \caption{USPS: external quality metrics.}
    \label{tab:usps-ext}
\end{table}

\paragraph{USPS}
In this smaller dataset, the runtime results of \cref{tab:runtime}---while in any case small---are preferable for \HDBSCAN, as the advantages brought by asymptotical complexity are irrelevant here. Results in \cref{tab:usps-ext} are, on the other hand, quite interesting: once again, the regularization effects discussed in \cref{sec:algorithm} improve the quality metrics on the results. In particular, AMI and ARI are both equal to 1, showing that \FISHDBC always returns two clusters: one for each of the two labels in the original dataset (AMI*/ARI* values are still lower than 1 because many digits are still considered as noise). On the other hand, \HDBSCAN returns a larger number of clusters (11), and some of them contain mixed labels.

\paragraph{Summary}
\FISHDBC enables performant clustering in cases where computing the full distance matrix falls short. 
Moreover, \FISHDBC rarely fares worse than \HDBSCAN in terms of quality metrics---in various cases, indeed, regularization effects improve result quality.

\begin{figure}[t]
    \centering
    \includegraphics[width=\columnwidth]{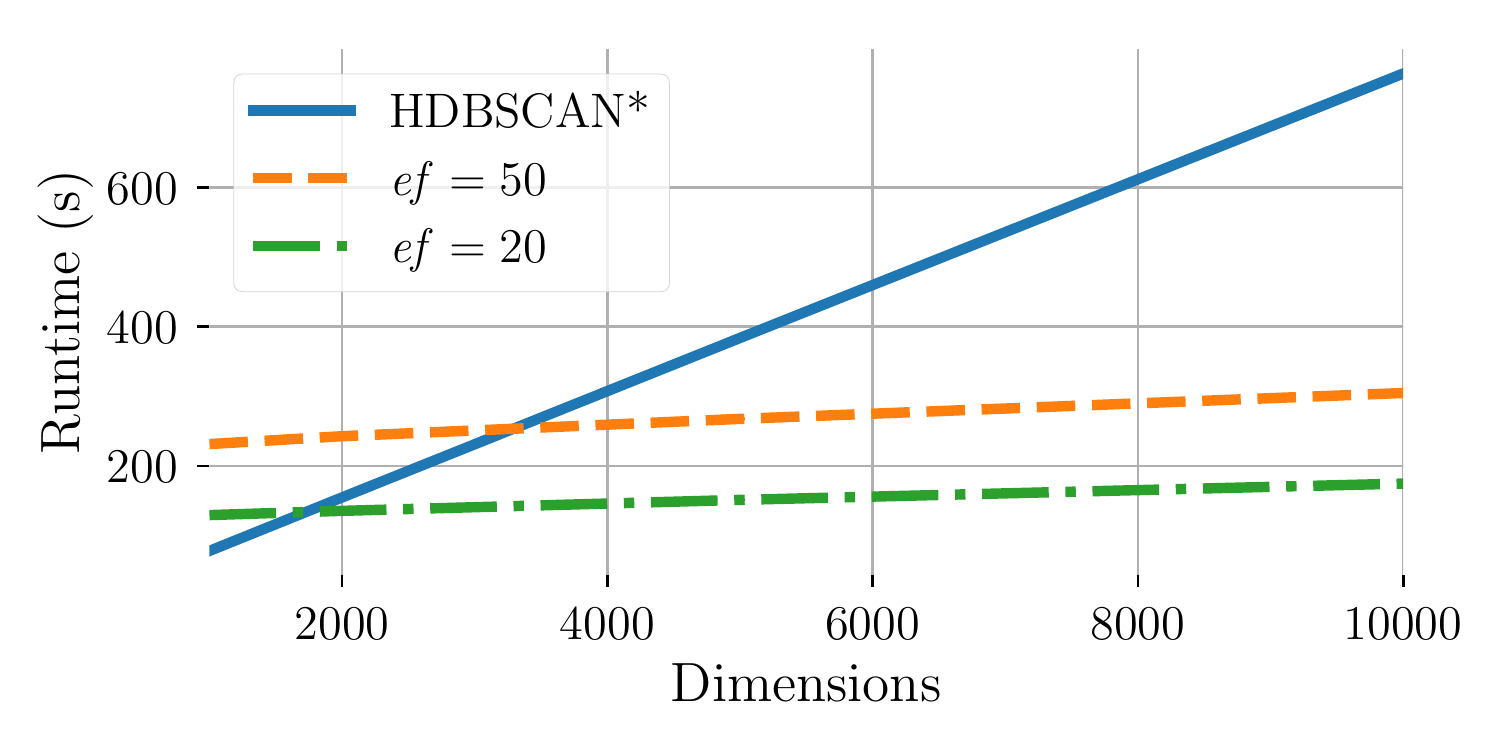}
    \caption{Blobs: runtime comparison.}
    \label{fig:blobs_runtime}
\end{figure}

\begin{table}[t]
    \centering
    \rowcolors{3}{gray!15}{white}
    \begin{tabular}{rrrrrrr}
    \toprule
    \multirow{2}{*}{Dimensions} & \multicolumn{2}{c}{$\ef=20$} & \multicolumn{2}{c}{$\ef=50$} & \multicolumn{2}{c}{\HDBSCAN} \\
    & AMI* & ARI* & AMI* & ARI* & AMI* & ARI* \\
    \midrule
    1\ 000 & 0.98 & 0.99 & 0.99 & 0.99 & \bf 1 & \bf 1 \\
    2\ 000 & 0.98 & 0.99 & 0.99 & 0.99 & \bf 1 & \bf 1 \\
    5\ 000 & 0.98 & 0.99 & 0.99 & 0.99 & \bf 1 & \bf 1 \\
    10\ 000 & 0.98 & 0.99 & 0.98 & 0.99 & \bf 1 & \bf 1 \\
    \bottomrule
    \end{tabular}
    \caption{Blobs: external quality metrics.}
    \label{tab:blobs-ext}
\end{table}

\begin{table*}[htbp]
    \newcommand{\cfill}{\xrfill[2.5pt]{0.05pt}}
    \newcommand{\oom}{\multicolumn{7}{c}{\cfill\textit{Out of memory}\cfill}}
    \centering
    \begin{adjustbox}{max width=\textwidth}
    \begin{tabular}{lrrrrrrrrr}
        \toprule
        \multirow{2}{*}{Dataset} & \multirow{2}{*}{Size} & \multirow{2}{*}{Algorithm (\ef)} & \multicolumn{2}{c}{Clustered elements} & \multicolumn{2}{c}{Clusters} & \multirow{2}{*}{Silhouette} & \multicolumn{2}{c}{Average distance} \\
        \cmidrule(lr){4-5} \cmidrule(lr){6-7} \cmidrule(lr){9-10}
        & & & flat & hierarchical & flat & hierarchical & & intra-cluster & inter-cluster  \\
        \midrule
        & 
           & \FISHDBC (20) & \bf 398 & \bf 1\ 546 &     4 &    10 &     0.509 &     0.381 & \bf 0.882 \\
         & & \FISHDBC (50) &     385 &        995 &     3 &     6 &     0.513 &     0.381 &     0.871 \\
        \multirow{-3}{*}{DW-Kos}  & \multirow{-3}{*}{3\ 430} &  \HDBSCAN &     353 &        353 & \bf 2 & \bf 4 & \bf 0.532 & \bf 0.375 &     0.854 \\
        \midrule
        \multirow{3}{*}{DW-Enron} & \multirow{3}{*}{39\ 861}
           & \FISHDBC (20) &      6\ 094 &     36\ 039 & \bf 222 & \bf 454 & \bf 0.552 & \bf 0.301 &     0.969 \\
         & & \FISHDBC (50) &      6\ 340 &     34\ 408 &     238 &     486 &     0.549 &     0.309 &     0.969 \\
         & &  \HDBSCAN & \bf  7\ 206 & \bf 39\ 344 &     299 &     642 &     0.469 &     0.326 & \bf 0.973 \\
        \midrule
        \multirow{3}{*}{DW-Nytimes} & \multirow{3}{*}{300\ 000} 
           & \FISHDBC (20) &     29\ 546 &     299\ 729 & \bf 802 & \bf 1\ 754 & \textit{OOM} &     0.552 &     0.967 \\
         & & \FISHDBC (50) & \bf 31\ 404 & \bf 299\ 757 &     888 &     1\ 924 & \textit{OOM} &     0.552 & \bf 0.968 \\
         & &  \HDBSCAN & \oom \\
        \midrule
        \multirow{3}{*}{Finefoods} & \multirow{3}{*}{568\ 464}
           & \FISHDBC (20) &      77\ 152 &     566\ 484 & \bf 2\ 924 & \bf 6\ 262 & \textit{OOM} &     0.282 & \bf 0.372 \\
         & & \FISHDBC (50) &      79\ 904 &     568\ 104 &     3\ 531 &     7\ 486 & \textit{OOM} & \bf 0.226 &     0.363 \\
         & &  \HDBSCAN & \oom \\
        \midrule
        \multirow{3}{*}{Household} & \multirow{3}{*}{2\ 049\ 280}
           & \FISHDBC (20) &     1\ 587\ 223 &     2\ 049\ 175 &     12\ 268 & \bf 61\ 582 & \textit{OOM} &     2.71 &     13.48 \\
         & & \FISHDBC (50) & \bf 1\ 649\ 304 &     2\ 049\ 224 & \bf 11\ 198 &     61\ 902 & \textit{OOM} &     2.76 &     13.17 \\
         & &  \HDBSCAN &     1\ 395\ 980 & \bf 2\ 049\ 273 &     53\ 358 &    173\ 198 & \textit{OOM} &    \bf 2.41 & \bf 13.94 \\
        \bottomrule
    \end{tabular}
    \end{adjustbox}
    \caption{Internal clustering quality metrics. \textit{OOM} stands for out-of-memory errors when computing the Silhouette metric.}
    \label{tab:internal}
\end{table*}

\begin{table}[htbp]
    \centering
    \begin{adjustbox}{max width=\columnwidth}
    \rowcolors{3}{gray!15}{white}
    \begin{tabular}{lrrrrr}
    \toprule
    \multirow{2}{*}{Dataset} & \multicolumn{2}{c}{$\ef=20$} & \multicolumn{2}{c}{$\ef=50$} & \HDBSCAN \\
    & build & cluster & build & cluster & (accelerated?) \\
    \midrule
    Blobs & \multicolumn{5}{c}{\textit{see \Vref{fig:blobs_runtime}}} \\
    DW-Kos & 27.4 & 0.102 & 37.1 & 0.103 & \textbf{1.06} (yes) \\
    DW-Enron & 616 & 2.39 & 851 & 2.06 & \textbf{112} (yes) \\
    DW-NYTimes & \textbf{8\ 733} & 41.1 & 12\ 604 & 36.8 & \textit{OOM} (yes) \\
    Finefoods & \textbf{50\ 422} & 48.9 & 84\ 765 & 42.9 & \textit{OOM} (no) \\
    Fuzzy hashes & \multicolumn{5}{c}{\textit{see \Vref{fig:fuzzy_runtime}}} \\
    Household & 27\ 375 & 123 & 38\ 759 & 109 & \textbf{24\ 258} (yes) \\
    Synth & \multicolumn{5}{c}{\textit{see \Vref{tab:synth_runtime}}} \\
    USPS & 9.1 & 0.0500 & 12.1 & 0.0502 & \textbf{1.57} (no) \\
    \bottomrule
    \end{tabular}
    \end{adjustbox}
    \caption{Runtime (in seconds).}
    \label{tab:runtime}
\end{table}

\subsection{\FISHDBC Versus Accelerated \HDBSCAN}

We now consider Euclidean and cosine distance, for which \HDBSCAN provides a high-performance accelerated implementation.

\paragraph{Blobs}
These datasets have between 1,000 and 10,000 dimensions. \HDBSCAN uses a KD-tree here, but as the number of dimensions grows the effectiveness of such data structures decreases. In \cref{fig:blobs_runtime}, we see how the computation for \HDBSCAN increases quite steeply as dimensionality grows; on the other hand, growth is definitely slower for \FISHDBC thanks to the lower cost of approximated search through HNSWs.

Quality metrics in \cref{tab:blobs-ext} show that, here, \FISHDBC pays a small price in terms of clustering quality. Here, the experiment was repeated on 30 randomly generated datasets for each number of dimensions, and the standard deviation in AMI* and ARI* is, in all cases, 0.01 for \FISHDBC and 0 for \HDBSCAN.

\paragraph{Household}

In this 7-dimensional Euclidean dataset, one may speculate that \FISHDBC would be largely outperformed by the accelerated ad-hoc \HDBSCAN implementation (it uses an elaborate dual-tree version of Borůvka's algorithm). Actually, as reported in \cref{tab:runtime}, \HDBSCAN is only slightly faster than \FISHDBC. It is possible that optimizations on constant factors, e.g., swapping our pure Python HNSW implementation with a faster one, could make \FISHDBC faster in this case as well.
Intra- and inter-cluster quality metrics (\cref{tab:internal}) are better for \HDBSCAN, but \FISHDBC produces a smaller number of clusters, which is arguably more desirable for data exploration because the summarization due to clustering is more succinct.
While a considerable number of elements are categorized as noise in the flat clustering, almost all elements end up in a cluster when we consider the hierarchical clustering, which can facilitate data exploration tasks. This benefit is shared by both \FISHDBC and \HDBSCAN, for most datasets reported in \cref{tab:internal}.

\paragraph{Docword}

We conclude our evaluation by examining sparse vector datasets where we use cosine distance, which has an accelerated ad-hoc implementation in \HDBSCAN~\cite{mcinnes2017hdbscan}.
Internal quality metrics in \cref{tab:internal} are again similar between \FISHDBC and \HDBSCAN. Results on runtime in \cref{tab:runtime}, however, are quite different: the lookup structures of \HDBSCAN result in faster execution but larger memory footprint; hence, \FISHDBC can compute results for DW-NYTimes while \HDBSCAN fails with an out-of-memory error.

\paragraph{Summary}
Ad-hoc lookup structures are appealing, but they do not always outperform the generic acceleration of \FISHDBC. \FISHDBC outperforms \HDBSCAN in very high-dimensional dense datasets like Blobs, and because of its lower memory footprint it can handle dataset that \HDBSCAN cannot like DW-NYTimes.

\section{Conclusion}
\label{sec:discussion}
\label{sec:conclusion}

\FISHDBC can deal with arbitrary distance functions and can handle datasets that are too large for our \HDBSCAN reference. Its core features are providing cheap, incremental computation while supporting arbitrary data and distance functions, avoiding $\bigO(n^2)$ complexity without needing filter functions or lookup indices: domain experts are free to write arbitrarily complex distance functions reflecting the quirks of the data at hand. In addition to being incremental, scalable and flexible, \FISHDBC supports hierarchical clustering. It is also an option for very high-dimensional datasets where lookup structures suffer from the curse of dimensionality: our results shows that for datasets that have very high dimensionality \FISHDBC can outperform ad-hoc accelerated approaches.

We believe that separating neighbor discovery from incremental model maintenance is a powerful approach, which allows for algorithms that are easier to reason about, implement and improve.

\balance
\bibliographystyle{ACM-Reference-Format}
\bibliography{bibliography}

\end{document}